\documentclass[11pt,journal,final,letterpaper,onecolumn]{article}
\usepackage[top=1in, bottom=1in, left=1in, right=1in]{geometry}
\usepackage{setspace}
\usepackage[round]{natbib}

\usepackage{algorithm}
\usepackage{algorithmic}

\usepackage{breqn}
\usepackage[utf8]{inputenc} 
\usepackage[T1]{fontenc}    
\usepackage{hyperref}       
\usepackage{url}            
\usepackage{booktabs}       
\usepackage{amsfonts}       
\usepackage{nicefrac}       
\usepackage{microtype}      

\usepackage{amsmath, amsthm,amssymb} 
\usepackage{enumerate}

\usepackage{wrapfig}
\usepackage{caption}
\usepackage{subcaption}
\usepackage{appendix}
\usepackage{color}

\newcommand{\beq}{\vspace{0mm}\begin{equation}}
\newcommand{\eeq}{\vspace{0mm}\end{equation}}
\newcommand{\beqs}{\vspace{0mm}\begin{eqnarray}}
\newcommand{\eeqs}{\vspace{0mm}\end{eqnarray}}
\newcommand{\barr}{\begin{array}}
\newcommand{\earr}{\end{array}}

\newcommand{\Lmat}[0]{{{\bf L}}}

\newcommand{\Xmat}[0]{{{\bf X}}}
\newcommand{\Ymat}[0]{{{\bf Y}}}

\newcommand{\vv}{\boldsymbol{v}}

\newcommand{\xv}{\boldsymbol{x}}

\newcommand{\cdotv}{\boldsymbol{\cdot}}

\newcommand{\tdag}{\textsuperscript{\textdagger}}
\newcommand{\tddag}{\textsuperscript{\textdaggerdbl}}
\newcommand{\tstar}{\textsuperscript{\textasteriskcentered}}

\newcommand{\betav}[0]{{\boldsymbol{\beta}}}

\newcommand{\thetav}{\boldsymbol{\theta}}

\newcommand{\nuv}[0]{{\boldsymbol{\nu}}}

\newcommand{\piv}{\boldsymbol{\pi}}

\DeclareMathOperator*{\argmin}{argmin} 
\usepackage{mathtools}

\usepackage{multicol}
\usepackage{graphicx}
\usepackage{adjustbox}

\DeclarePairedDelimiter{\norm}{\lVert}{\rVert}

\newcommand{\given}{\,|\,}

\newtheorem{theorem}{Theorem} 
\newtheorem{lemma}[theorem]{Lemma}

\title{Convex Polytope Trees}

\usepackage{authblk}

\author[1]{Mohammadreza Armandpour \thanks{armand@stat.tamu.edu}}
\author[2]{ Mingyuan Zhou \thanks{mingyuan.zhou@mccombs.utexas.edu}}
\affil[1]{Department of Statistics, Texas A\&M University}
\affil[2]{McCombs School of Business, The University of Texas at Austin}

\begin{document}

\maketitle

\begin{abstract}
A decision tree is commonly restricted to use a single hyperplane to split the covariate space at each of its internal nodes. It often requires a large number of nodes to achieve high accuracy, hurting its interpretability. In this paper, we propose convex polytope trees (CPT) to expand the family of decision trees by an interpretable generalization of their decision boundary. The splitting function at each node of CPT is based on the logical disjunction of a community of differently weighted probabilistic linear decision-makers, which also geometrically corresponds to a convex polytope in the covariate space. We use a nonparametric Bayesian prior at each node to infer the community's size, encouraging simpler decision boundaries by shrinking the number of polytope facets. We develop a greedy method to efficiently construct CPT and scalable end-to-end training algorithms for the tree parameters when the tree structure is given. We empirically demonstrate the efficiency of CPT over existing state-of-the-art decision trees in several real-world classification and regression tasks from diverse domains.

\end{abstract}

\section{INTRODUCTION}


Decision trees \citep{breiman1984classification} are highly interpretable models,  which make them favorable in high-stakes domains such as medicine \citep{valdes2016mediboost, martinez2018circadian}, criminal justice \citep{bitzer2016analyse}, and network analysis \citep{backstrom2006group, armandpour2019robust}. They are also resistant, if not completely immune, to the inclusion of many irrelevant predictor variables. However, trees usually do not have high accuracy, which somewhat limits their use in practice.
Current main approaches to improve the performance of decision trees are making large trees or using ensemble methods \citep{dietterich2000experimental, dietterich2000ensemble,hastie2009elements,zhou2012ensemble}, such as bagging 
\citep{breiman2001random} and boosting \citep{drucker1996boosting, freund1997decision}, 
which all come with the price of harming model interpretability. There is a trade-off challenge between the accuracy and interpretability of a decision tree model.





Prior work has attempted to address the aforementioned challenge and improve the performance of trees by introducing oblique tree models \citep{heath1993induction, murthy1994system}. These families of models are generalizations of classical trees, where the decision boundaries are hyperplanes, each of which is not constrained to be 
axis-parallel and can have an arbitrary orientation. This change in the decision boundaries has shown to reduce the size of the trees. However, the tree size is still often too large in a real dataset to make it amenable to interpretation. 
 There has been an extensive body of research to improve the training of the oblique trees and enhance their performance \citep{wickramarachchi2016hhcart, bertsimas2017optimal,carreira2018alternating, lee2019locally}, yet their large size remains a challenge.

In this paper, we propose convex polytope decision trees (CPT) to expand the class of oblique trees by extending hyperplane cuts to more flexible geometric shapes. To be more specific, the decision boundaries induced by each internal node of CPT are based on noisy-OR \citep{pearl2014probabilistic} of multiple linear classifiers. And since noisy-OR has been widely accepted as an interpretable Bayesian model \citep{richens2020improving}, our generalization keeps the interpretability of oblique trees intact. Furthermore, CPT's decision boundaries geometrically resemble a convex polytope ($i.e.$, high dimensional convex polygon). Therefore, the decisions at each node have both logical and geometrical interpretation.
We use the gamma process \citep{ferguson1973bayesian, kingman1992poisson,zhou2018parsimonious}, a nonparametric Bayesian 
prior, to infer the number of polytope facets adaptively at each internal tree node and regularize the capacity of the proposed CPT. A realization of the gamma process consists of countably infinite atoms, each of which is used to represent a weighted hyperplane of a convex polytope. The shrinkage property of the gamma process helps us to encourage having simpler decision boundaries, therefore help resist overfitting and improve interpretability.

The training of CPT, like that of oblique trees, is a challenging task because it requires learning both the structure ($i.e.$, the topology of the tree and the cut-off for the decision boundaries) and the parameters ($i.e.$, parameters of noisy-OR). The structure is a discrete optimization problem, involving the search over a potentially large problem space. In this work, we present two fully differentiable approaches for learning CPT models, one based on mutual information maximization, applicable for both binary and multi-class classification, and the other based on variance minimization, applicable for regression. The differentiable training allows one to use modern stochastic gradient descent (SGD) based programming frameworks and optimization methods for learning the proposed decision trees for both classification and regression.

Experimentally, we compare the performance of CPT to state-of-the-art decision tree algorithms \citep{carreira2018alternating, lee2019locally} on 
a variety of representative
regression and classification tasks. We experiment with several real-world datasets from diverse domains, such as computer vision, tabular data, and chemical property data. Experiments demonstrate that CPT outperforms state-of-the-art methods with higher accuracy and smaller size. 

Our main contributions include: 1) We propose an interpretable generalization to the family of oblique decision trees models; 2) We regularize the expressive power of CPT, using a nonparametric Bayesian shrinkage prior for each node split function; 3) We provide two scalable and differentiable ways of learning CPT models, one for classification and the other for regression, which efficiently search for the optimal tree; 4) We experimentally evaluate CPT on several different types of predictive tasks, 
illustrating that this new approach outperforms the prior work in having  higher
accuracy achieved with a smaller size.

\section{RELATED WORK}
Interperatble inference is of utmost importance (e.g. \citet{ahmad2018interpretable, zhang2018interpretable, sadeghian2019drum}), among which decision trees are popular. Most of the literature on decision trees has been focused on how to train a single tree or an ensemble of multiple ones \citep{hastie2009elements}. There has been little work in making decision boundaries more flexible. One of the reasons for this lack of research is the fact that the computational complexity of the problem increases even for simple generalization of the decision boundaries. For example, with $N$ as the number of data points and $d$ as the dimension of the input space, the generalization of the coordinate wise to an oblique hyperplane cut, increases the number of possible splits of data points from $N d$ to $\sum_{i=0}^{d} {{N}\choose{i}}$ just for a single node \citep{vapnik16ja}.

Some methods perform hyperplane cuts in an extended feature space, created by concatenating the original features and newly generated ones~\citep{ahmad2014decision}, to get more flexible decision boundaries. These new features can be engineered or kernel-based and are not designed for interpretability, but to gain performance in an ensemble of such trees using Random Subspaces \citep{ho1998random}. We will follow this section by a literature review of the training algorithms for (oblique) trees.

Conventional methods for decision tree induction are greedy, where they grow the tree nodes one at a time. The greedy construction of oblique trees can be done by using coordinate descent to learn the parameters of each split \citep{murthy1994system}, or by a projection of the feature space to a lower dimension then using coordinate-cut \citep{menze2011oblique, wickramarachchi2016hhcart}. However, the greedy procedure often leads to sub-optimal trees. 

There have been several attempts to non-greedy optimization, which rely on either fuzzy or probabilistic split functions \citep{suarez1999globally, jordan1994statistical, kontschieder2015deep}. The probabilistic trees are sometimes referred to as soft decision trees \citep{frosst2017distilling} and have been applied to computer vision problems \citep{kontschieder2015deep, hehn2019end}. In these methods, the assignment of a single sample to the leaf is fuzzy or probabilistic, and gradient descent is used for the optimization of the tree. Most of these algorithms remain probabilistic at the test time, which consequently leads to uninterpretable models as the prediction for each sample will be based on multiple leaves of the tree instead of just one. There are no probabilistic trees in the literature developed for the regression task to the best of our knowledge.

Other advances towards the training of an oblique tree are based on constructing neural networks that reproduce decision trees \citep{yang2018deep, lee2019locally}.  \citet{yang2018deep} use a neural network with argmax activations for the representation of classic decision trees with coordinate cuts, but they are not scalable to high-dimensional data. \citet{lee2019locally} use the gradient of a ReLU network with a single hidden unit at each layer and skip-connections to construct an oblique decision tree. They achieve state-of-the-art results on some molecular datasets, but they have to make complete trees for a given depth and need high-depth trees.

In contrast to our method, there are other training algorithms that require the structure of the tree at the beginning. Some of the works in this direction like \citet{bennett1994global} and \citet{bertsimas2017optimal} use linear programming, or mixed-integer linear programming, to find a global optimum tree. Therefore, these methods are computationally expensive and not scalable. \citet{norouzi2015efficient} derive a convex-concave upper bound on the tree’s empirical loss and optimize that loss using SGD. A recent work \citep{carreira2018alternating} proposes tree alternating optimization, where one directly optimizes the misclassification error over separable subsets of nodes, achieving the state-of-the-art empirical performance on some datasets \citep{zharmagambetov2019experimental} .

We conclude this section by relating our splitting rule at each internal node to some relevant classification algorithms~\citep{aiolli2005multiclass, manwani2010learning, manwani2011polyceptron, wang2011trading, kantchelian2014large, zhou2018parsimonious}. The two most related works are convex polytope machine (CPM) \citep{kantchelian2014large} and infinite support hyperplane machine (iSHM)  \citep{zhou2018parsimonious}, which both exploit the idea of learning a convex polytope associated decision boundary. In particular, iSHM can be considered as a decision stump (a tree of depth one) of CPT. iSHM is like a single hidden layer NN, which provides different values for different input features (at test time, therefore less interpretable) and is restricted to the binary classification task. However, CPT's decision function at test time, at each node, just assigns two values to each feature space to send the data to the right or left. We stack those binary classifiers in the form of a decision tree. This provides a locally constant function for any task with a differentiable objective. Thus it is not necessarily restricted to the binary classification task.

\section{CONVEX POLYTOPE TREE AND ITS INFERENCE ALGORITHMS}

Suppose we are given the training data $(\Xmat, \Ymat)= \{(\xv_n, y_n)\}_{n=1}^N$, where pairs of $(\xv_n, y_n)$ are drawn independently from an identical and unknown distribution $D$. Each $\xv_n \in \mathbb{R}^d$ is a $d$-dimensional data with a corresponding label $y_n \in \mathcal{Y}$. In the classification setting, $\mathcal{Y}=\{1, \cdots, K\}$
and in regression scenario $\mathcal{Y}= \mathbb{R}$. The aim is to learn a function
 $F:\mathbb{R}^d \rightarrow  \mathcal{Y}$ that will perform well in predicting the label on samples from $D$.
 
Decision tree methods construct the function $F$ by recursively partitioning the feature space to yield a number of hierarchical, disjoint regions, and assign a single label (value) to each region. The final tree is comprised of branch and leaf nodes, where the branch nodes make the splits and leaf nodes assign values to each related region. For both classical~\citep{breiman1984classification,quinlan1986induction,quinlan2014c4} and oblique trees~\citep{murthy1993oc1, murthy1994system}, the decision boundary at each branch node can be expressed as whether $\betav \xv>0$ or $\betav \xv \leq 0$. We do not explicitly consider the bias term in the decision boundary because we assume, for the sake of notational brevity, that $\xv$ includes a constant unitary component corresponding to a bias term. The $\betav$, in the case of classical trees, is limited to having just one coordinate equal to one and the rest equal to zero, other than the coordinate corresponding to bias. However, oblique trees do not make any restriction on~$\betav$. In what follows, we will explain how we move beyond 
the oblique decision trees.

\subsection{Convex Polytope Constrained Decision Boundary}

By extending the idea of disjunctive interaction (noisy-OR, also commonly referred to as probabilistic-OR) \citep{shwe1991probabilistic, jaakkola1999variational, zhou2018parsimonious} from probabilistic reasoning to the decision tree problem, we make the decision boundaries more flexible while preserving interpretability. To that end, we transform the problem of a node splitting, right or left, to a committee of experts that make individual binary decisions (``Yes'' or ``No''). 
Note the probabilistic-OR construction shown below, while being closely related to iSHM \citep{zhou2018parsimonious}, is distinct from it in that the Yes/No decisions are latent rather than observed binary labels. 
The committee votes ``Yes'' if and only if at least one expert votes ``Yes'', otherwise votes ``No''. Thus, the final vote at each node is
$$
vote = \bigvee\nolimits_{i=1}^K vote_i,
$$
where $\bigvee$ denotes the logical OR operator. We model each expert as a linear classifier who votes ``Yes'' with probability
\beq \label{eq:each_person}
\textstyle P({vote}_i = \text{``Yes''}\given \{r_i,\betav_i \} ,\xv ) = 1 - ({1+e^{\betav_i'\xv}})^{-r_i} ,
\eeq
where $r_i\geq 0$ and $\betav_i \in \mathbb{R}^d$ are parameters of expert $i$. Now assuming that each expert votes independently, we can express the probability of the committee voting ``Yes'' as
\begin{dmath}\label{eq:com_v}
\textstyle P(vote =\text{``Yes''}\given \{r_i,\betav_i\}_i,\xv) = 1-\prod\nolimits_{i=1}^K(1-p_{i})=1-e^{-\sum_{i=1}^{K}r_i \ln{(1+e^{\betav_i'\xv})}},
\end{dmath}
where $p_i$ is the probability of expert $i$ voting ``Yes'' and $K$ is the total number of experts. 
We can now define the split function at each node by thresholding the committee voting probability:
$$A_{\textit{left}}:=\{ \xv \given \xv \in \mathbb{R}^d, \textstyle P(vote =\text{``Yes''}\given \{r_i,\betav_i\}_i,\xv)\leq q_{\textit{thr}}\}$$

where $A_{\textit{left}}$ and $A_{\textit{right}} :=\mathbb{R}^d \setminus A_{\textit{left}}$ are the related splits of the space.

To elaborate on the geometric shape and interpretability of the decision boundaries,~consider $K=1$ ($i.e.$ a single expert). In this scenario, the decision boundary becomes a hyperplane, which is perpendicular to $\betav$. In fact, the probability function for each expert is based on the signed distance of $\xv$ to the hyperplane perpendicular to $\betav$. And, parameter $r$ controls how smoothly the probability transitions from 0 to 1, where a larger $r$ leads to sharper changes. This class of models with $K=1$ and $r=1$ are identical to oblique trees, which are interpretable. The interpretability of $K\geq 1$ is provided by the fact that linear classifiers and probabilistic-OR operation are interpretable \citep{richens2020improving}.

To geometrically analyze the implied decision regions, we provide the following theorem.
\begin{theorem}\label{thm:polytope}
For any $\{r_i,\betav_i\}_{i=1}^K$, such that $r_i \in \mathbb{R}_+ $ and $\beta_i \in \mathbb{R}^d $, let: 
$$
A_{\textit{left}}=\{ \xv \given \xv \in \mathbb{R}^d,  ~~ \textstyle f(\xv)\leq q_{\textit{thr}} \} ,
$$
where:
\beq
\textstyle f(\xv) = 1-e^{-\sum_{i=1}^{K}r_i \ln{(1+e^{\betav_i'\xv})}},
\eeq
then $A_{\textit{left}}$ is a convex set, confined by a convex polytope.
\end{theorem}
The proof is provided in the Appendix.

The above theorem 
shows for $K\geq 1$, the decision region ($A_{left}$) is a convex set confined by a convex $K$-sided polytope. More precisely, each facet of the convex polytope is a hyperplane corresponding to an expert perpendicular to its $\betav$. Also worth noting, an expert with a larger $r$ has more effect on the decision boundary, making sharper changes to the probability function. This can also be perceived as the value of their decision in the committee. Therefore, our method not only has a strong relationship with probabilistic-OR type models that provide interpretability for the model parameters \citep{almond2015bayesian}, but also has decision boundaries with interpretable geometric characteristics. We propose a class of models, Convex Polytope Trees (CPT), where each node of the tree follows the above splitting function.

\subsection{Gamma Process Prior } \label{sec:gamma}

To regularize CPT, and motivate simpler decision boundaries we use a nonparamteric Bayesian shrinkage prior. Specifically, we put the gamma process prior  \citep{ferguson1973bayesian,kingman1992poisson,zhou2018parsimonious} on the splitting function of each node in the tree. Each realization of the gamma process, consisting of countably infinite weighted atoms whose total weight is a finite gamma random variable, can~be~described~as~\beq G=\sum\nolimits_{i=1}^\infty r_i \delta_{\betav_i} , ~~\text{such that} ~~\betav_i \in  \mathbb{R}^d, r_i \in \mathbb{R}_+ \eeq  

where $\betav_i$ represents an atom with weight $r_i$. More details about the gamma process can be found in \citet{kingman1992poisson}.
We put the prior on the CPT by considering $\betav_i$ and $r_i$ as the parameters of the splitting function related to equation \eqref{eq:com_v}. Due to the gamma process's inherent shrinkage property, just a small finite number of experts will have non-negligible weights $r_i$ at each node. This behavior encourages the model to have simpler decision boundaries ($i.e.$ smaller number of experts or equivalently fewer polytope facets) at each node. This improves the interpretability and regularization of the model. The gamma process allows a potentially infinite number of experts at each node. For the convenience of implementation, we truncate the gamma process to a large finite number of atoms.

To further encourage simpler models at each node of the tree, we also put a shrinkage prior on $\betav$ of each expert. In particular, we consider the prior:
\beq \betav_{i}\sim\prod\nolimits_{j=0}^{d}\textstyle{\int} \mathcal{N}(\beta_{ji};0,\sigma_{ji}^{2}) 
\mbox{InvGamma}
(\sigma_{ji}^{2};a_\beta,1/b_{\beta i})d \sigma_{ji}^{2}\eeq
and $b_{\beta i}\sim \mbox{Gamma}(e_0,1/f_0)$
which motivates sparsity due to the InvGamma distribution on the scale parameter~\citep{tipping2001sparse, zhou2018parsimonious}.

\subsection{Training Algorithm}

Finding an optimal CPT requires solving a combinatorial, non-differentiable optimization problem. This is due to the large number of possibilities that any single node can separate the data. We propose a continuous relaxation of the splitting rule of each node to alleviate this computationally challenging task. Particularly, each internal node makes probabilistic rather than deterministic decisions to send samples to its right or left branch. We set the probability of going right the same as \eqref{eq:com_v}, or any monotonic function of it. We use this probabilistic version to train the tree in a differentiable manner. At the test time, we threshold the splitting functions to provide a deterministic tree. Below we explain in detail the proposed training algorithm for the parameters and structure of the tree.

\subsubsection{Learning Split Parameters}\label{sec:training}

Assuming the tree structure is given, we first explain how to infer the tree parameters.

For classification, we formulate the training as an optimization problem by considering the mutual information between the two random variables $\Ymat$ (category label) and $\Lmat$ (leaf id) as our objective function. This may seem similar to previous literature on learning a classical decision tree but it differs in two main ways: 1) we develop and optimize the mutual information for a probabilistic rather than deterministic tree, and 2) we  learn the parameters of all nodes jointly rather than learning them in~a~greedy~fashion.

We model our probabilistic tree by letting
\beq\label{eq:tree_main}
\begin{aligned}
\ell_n  \sim P_{\thetav}(\ell_n \given \xv_n) \quad \quad  \text{such that} \quad \quad \ell_n \in  S_{\textit{leaf}},
\end{aligned}
\eeq
where $S_{\textit{leaf}}$ is the set of all leaf nodes and $ P_{\thetav}(\ell_n \given \xv_n)$ is the probability of arriving at leaf $\ell$ given the sample feature $\xv_n$. We assume each internal node makes decisions independent of the others and use the probabilities in \eqref{eq:com_v} when sending a data sample to the left or right branch. This assumption lets us derive a formula for $P_{\thetav}$ as

\beq\label{eq:tree_par}
    P_{\thetav}(\ell_n \given \xv_n) =\prod\nolimits_{(v, d_v)  ~\in  \nuv_{\ell_n}} q_{\thetav_v}(d_v \given \xv_n ),
\eeq
where
\begin{align}
    q_{\thetav_v}(d_v = 1 \given \xv_n )&=  1 - q_{\thetav_v}(d_v = 0 \given \xv_n )\nonumber\\ &= 1-e^{-\sum_{i=1}^{K}r^{(v)}_i \ln{(1+e^{\betav^{(v)'}_i\cdot\xv})}}
\end{align}
and $\nuv_{\ell_n}$ is a path from the root to leaf $\ell_n$ and $d_{v} \in \{0,1\}$ encodes the right or left ($0$ or $1$) direction taken at node $v$. The mutual information between $\Ymat$ and $\Lmat$ can be expressed as
\begin{align}\label{eq:entropy}
\mathcal{I}(\Ymat, &\Lmat) = \mathcal{H}(\Ymat)- \mathcal{H}(\Ymat \given \Lmat) \nonumber \\ &= \mathcal{H}(\Ymat)- \sum\nolimits_{ \ell \in  S_{leaf}} p(\Lmat= \ell) \mathcal{H}(\Ymat \given \Lmat= \ell),
\end{align} 
where $\mathcal{H}(\cdotv)$ indicates the entropy of a random variable, and $\Ymat \given \Lmat $ follows a categorical distribution. Notice that, since $\mathcal{H}(\Ymat)$ does not depend on the tree parameters, to optimize mutual information, we only need to minimize $\mathcal{H}(\Ymat \given \Lmat)$. However, the evaluation of the conditional entropy term requires knowledge of the entire data distribution, thus we can not directly optimize~\eqref{eq:entropy}. 

To make the training possible, we approximate the true data distribution with the empirical one to get
\beq \label{eq:denom}
\hat{p}(\Lmat= \ell) = \textstyle \sum_{n=1}^N P_{\thetav}(\ell \given \xv_n)/N.
\eeq
Denote $\mathbf{1}_{[\cdotv]}$ as an indicator function.  By using Bayes' rule, we derive  $\hat{\piv}^{\ell}=(\hat{\pi}^{\ell}_{1}, \cdots, \hat{\pi}^{\ell}_{C})$,  the estimated probability vector of~the~categorical~distribution~for $\Ymat \given \Lmat= \ell$, as
\beq \label{eq:H}
\hat{\pi}^{\ell}_{j} = \frac{\sum_n \mathbf{1}_{[y_n=j]} P_{\thetav}(\ell \given \xv_n)}{\sum_n P_{\thetav}(\ell \given \xv_n)},~~j=1,\ldots,C\,\,\,.
\eeq

Now by using~\eqref{eq:H}, we can approximate the entropy term  $\mathcal{H}(\Ymat \given \Lmat= \ell)$ as
\beq
\mathcal{\hat{H}}(\Ymat \given \Lmat= \ell) = - \sum\nolimits_{j=1}^C \hat{\pi}^{\ell}_{j} \log (\hat{\pi}^{\ell}_{j}).
\eeq
Therefore, we can provide an estimator for the $\mathcal{H}(\Ymat \given \Lmat)$ as
\beq \label{eq:loss}
\mathcal{\hat{H}}(\Ymat \given \Lmat) = \sum\nolimits_{ \ell \in  S_{leaf}} \hat{p}(\Lmat= \ell) \mathcal{\hat{H}}(\Ymat \given \Lmat= \ell).
\eeq
By minimizing $\mathcal{\hat{H}}(\Ymat \given \Lmat)$ with respect to $\thetav$, we are in fact maximize the mutual information $\mathcal{\hat{I}}(\Ymat, \Lmat)$. As discussed in Section \ref{sec:gamma}, we also regularize CPT by adding a penalty term to \eqref{eq:loss}. We consider the negative log probability of the gamma process prior truncated with $K$ atoms~by~letting $$r_1,\ldots,r_K\stackrel{iid}\sim\mbox{Gamma}(\gamma_0/K,1/c_0)$$ where $r_1,\ldots,r_K$ are the parameters of internal nodes splitting function.

The penalty term for each internal node can be mathematically formulated as

\begin{dmath}
\centering
\textstyle\sum_{k=1}^K\left(-(\frac{\gamma_0}{K}-1)\ln r_k +c_0 e^{\ln r_k}\right)  +\textstyle(a_\beta+1/2)\sum_{j=0}^d\sum_{k=0}^K [ \ln(1+\beta_{jk}^2/(2b_{\beta k}) ) ]. 
\end{dmath}

The above procedure provides a differentiable way of learning branch node parameters, which dictates how a data sample will arrive at a leaf node. At the end of the training algorithm, we also need to assign the leaf node parameters, which determine how the tree predicts a sample. We pass the whole training set through the tree and assign the empirical distribution of all categories at each leaf node as its node parameters. This way of determining the leaf parameters has been shown  to achieve the highest AUC in binary-classification \citep{ferri2002learning}.

For regression, we replace the mutual information optimization by a variance reduction criteria. To be more specific, we learn the tree parameters with 
\beq\label{eq:regg}
\argmin_{\thetav} \{ \sum\nolimits_n P_{\thetav}(\ell \given \xv_n) (y_n-\hat{\mu}^{\ell})^2\}\eeq
such that
\beq
\hat{\mu}^{\ell} = \frac{\sum_n P_{\thetav}(\ell \given \xv_n) y_n}{\sum_n P_{\thetav}(\ell \given \xv_n)}
\eeq

and the calculation of $P_{\thetav}$ is exactly the same as in the classification case. The leaf parameters are set as the mean response of the points arriving at that leaf. 

Note for both scenarios, we use the threshold of 0.5 to do deterministic splits at the test time. Now that we know how to train a tree given its structure and how to use it at the test time, the next section will describe how to find an optimal topology and initial parameters. Algorithm 1 of the Appendix summarizes the training~method.

\subsubsection{Topology Learning}

We start by assuming the tree structure to be just a root node and its two child leaf nodes. We train this tree using the algorithm explained in Section~\ref{sec:training}. After training, we split the training set to two subsets (right and left), by thresholding the assigned probability. We calculate this threshold in the classification (regression) task, which achieves maximum mutual information (variance reduction) in the deterministic tree. To be more specific, for the classification task, we set the threshold $q_{\textit{thr}}$, such that it minimizes
\begin{align*}
\frac{n_0}{N}\cdot\mathcal{H}(\{y_{(i)}\}_{i=1}^{n_0})+ \frac{N-n_0}{N}\cdot \mathcal{H}(\{y_{(i)}\}_{i={n_0+1}}^{N}), \end{align*} where $$\quad p_{(n_0)}\leq q_{\textit{thr}} \leq p_{(n_0+1)}. 
$$
The $p_{(i)}$ is the $i$'th smallest probability assigned by the root node to the data samples. We further split each child node by considering it as root and applying the above algorithm using its data samples. We stop splitting a node with very few data samples and stop growing the tree when we reach a certain predefined maximum depth.

\begin{figure*}

\begin{subfigure}{0.15\textwidth}

  \includegraphics[width=\linewidth]{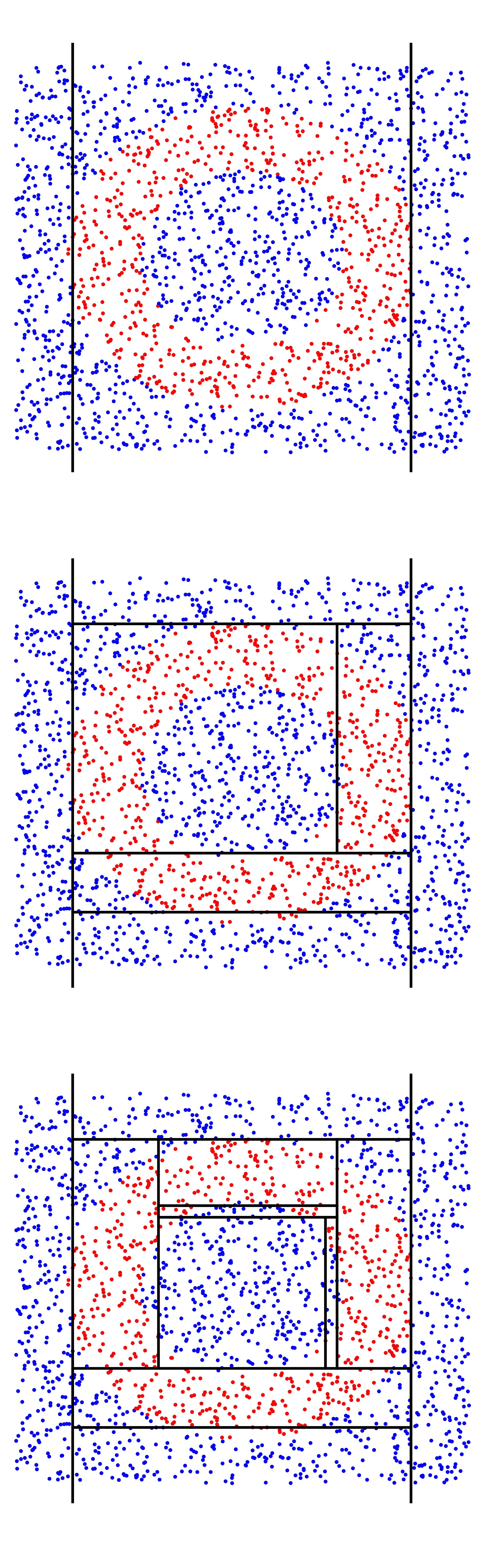}
    \caption{CART}
    \label{fig-a}
 
\end{subfigure}\hfill
\begin{subfigure}{0.14\textwidth}
  \includegraphics[width=\linewidth]{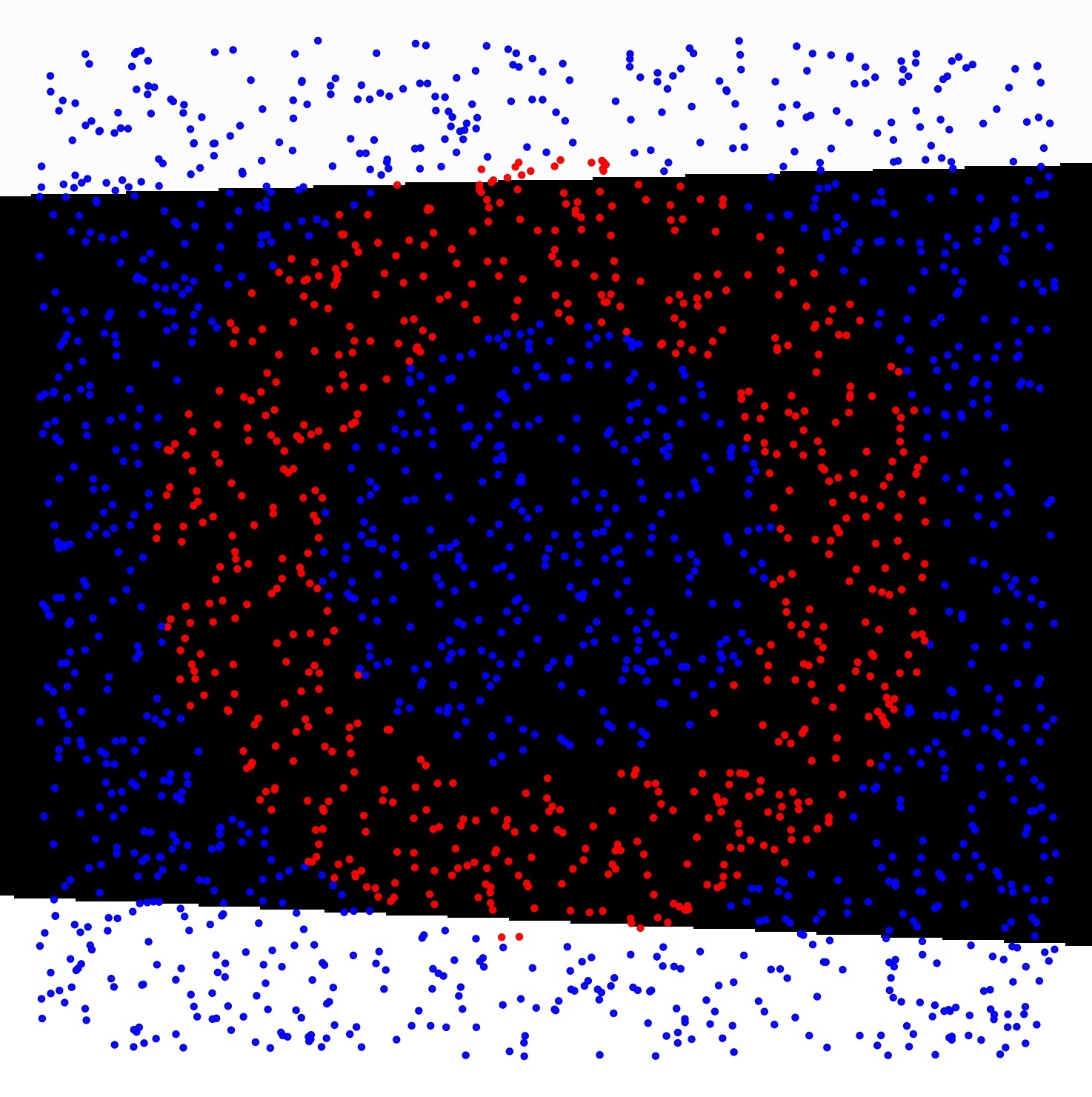}\vspace{2mm}
  \includegraphics[width=\linewidth]{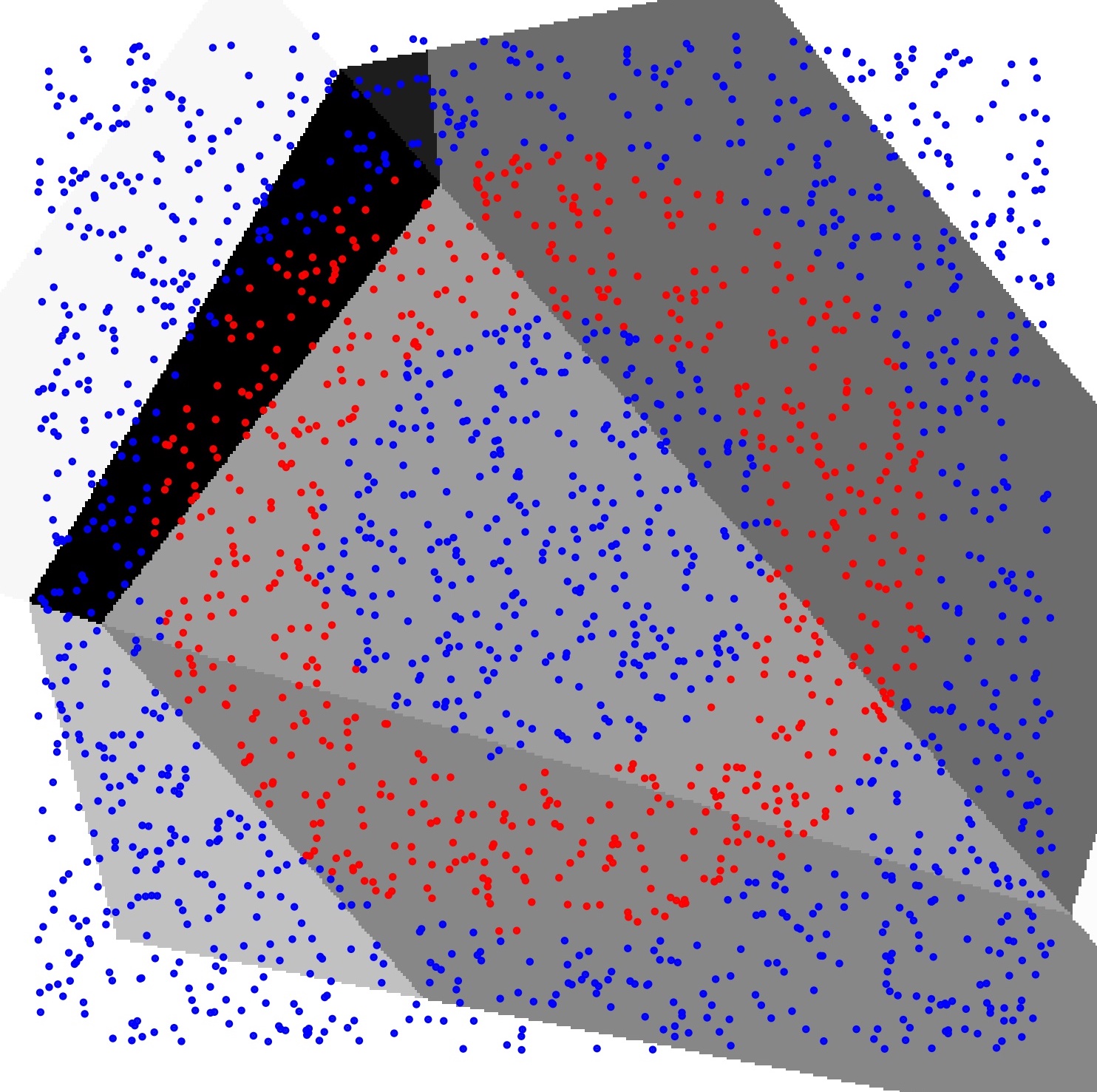}\vspace{2mm}
  \includegraphics[width=\linewidth]{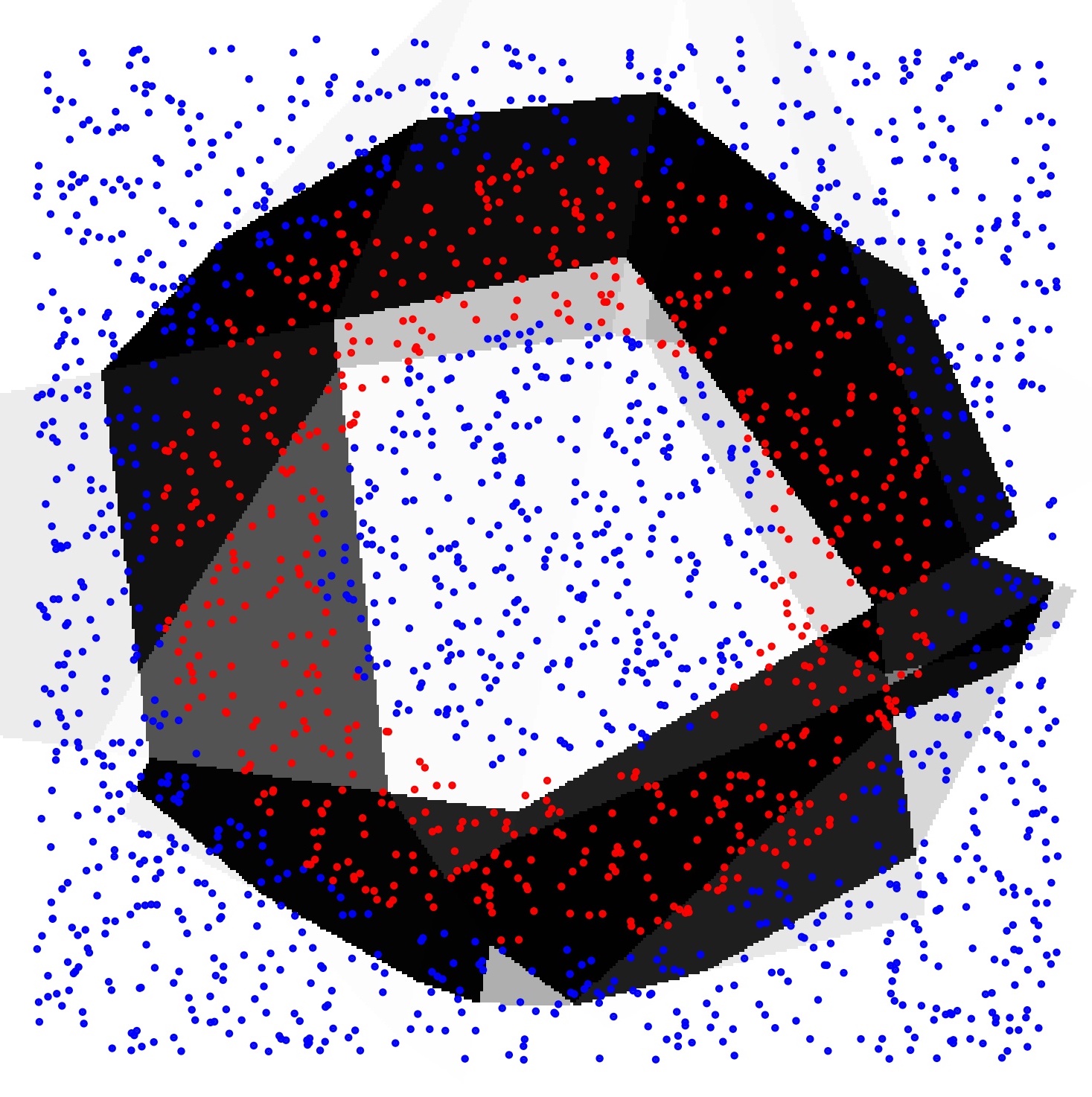}\vspace{3.5mm}
    \caption{LCN}
    \label{fig-b}
 
\end{subfigure}\hfill
\begin{subfigure}{0.48\textwidth}%
  \includegraphics[width=\linewidth]{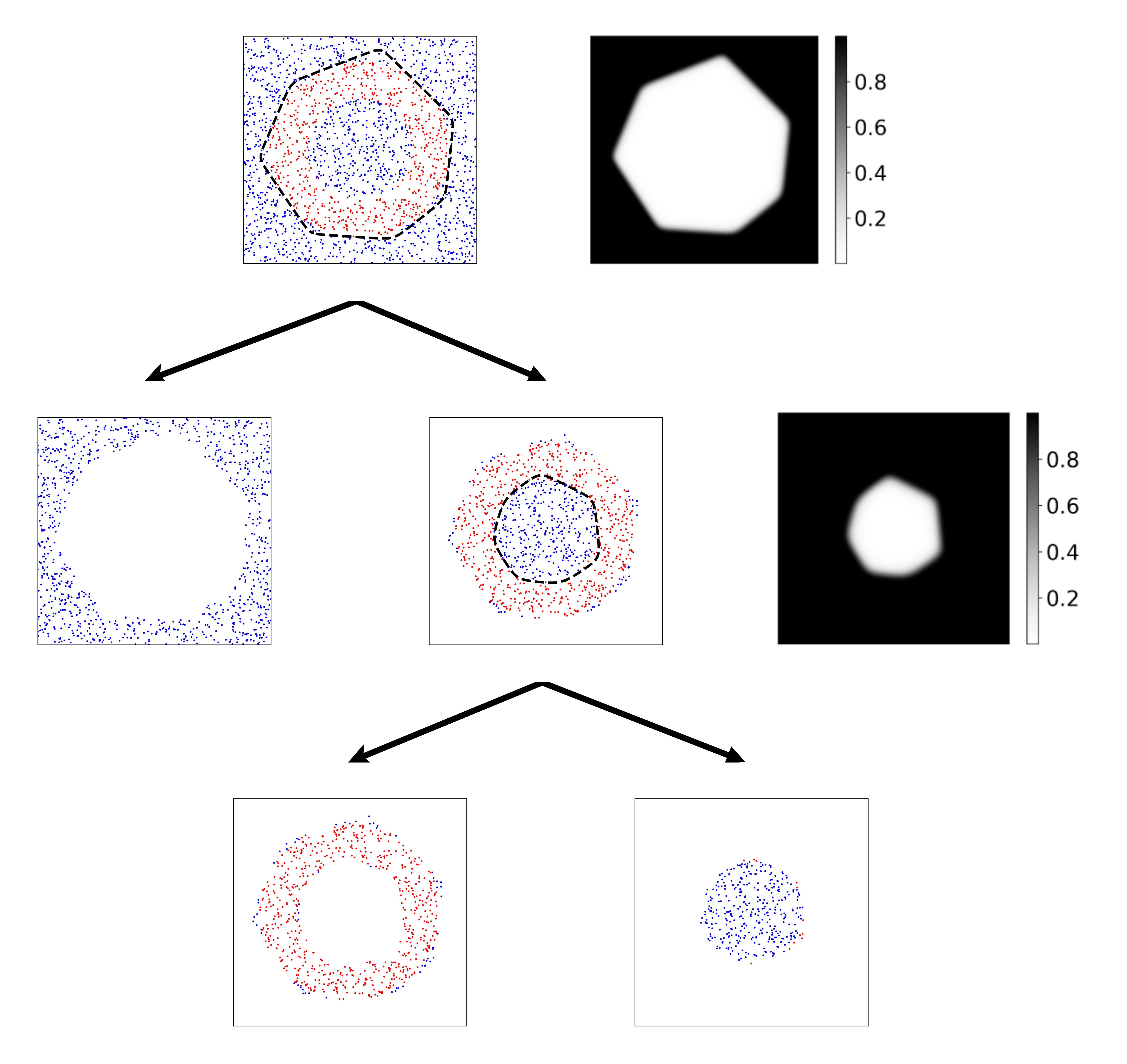}\vspace{2mm}
    \caption{CPT}
    \label{fig-c}
 
\end{subfigure}

\caption{Decision boundary visualizations of CART, LCN (depth 2, 6, and 10) and CPT (depth 2) on the synthetic 2D dataset. Figures \protect\subref{fig-a} and~\protect\subref{fig-b} show how CART and LCN partition the feature space for various depth. Figure \protect\subref{fig-c}, illustrates a CPT, visualizing each node by the data samples arrived at it. At each internal node, a heatmap of the corresponding probabilistic-OR function and the decision boundary (shown by a black dashed line) is presented.} 
\label{fig:syn}
\end{figure*}

Note that during the proposed tree-structure greedy training, we do not perform any parameter refining using the method presented in Section~\ref{sec:training}. The parameter refining at each step of adding a new node can further improve the accuracy, but it will come with the price of increased  computational complexity.

\section{EXPERIMENTS}

In this section, we empirically assess the qualitative and quantitative performances of CPT on datasets from various~domains. We show that CPT learns significantly smaller trees than its counterparts and makes more robust and accurate predictions. That is partly due to the high variance of the leaf node's prediction in classical (oblique) trees, resulting from fewer data samples in each partition. However, since CPT usually has fewer leaf nodes, each partition has a significant proportion of the dataset.

\begin{table*}[th]
  \caption{Dataset statistics}\label{tab:statistics}\vspace{-1mm}
\begin{adjustbox}{width=1\columnwidth,center}
  \centering
  \begin{tabular}{lcccccccccccccc}
    \toprule
    \small Dataset               & \small {MNIST} & \small {SensIT} & \small {Connect4} & \small {Letter}  & \small {PDBbind} & \small {Bace} & \small {HIV}  \\
    \midrule
    \small Task                  & \multicolumn{4}{c}{Multi-class classification} & \small Regression & \multicolumn{2}{c}{Binary classification}  \\
    \small Number of classes      & \small 10 & \small 3   & \small 3 & \small 26 & \small - & \small 1 & \small 1  \\
    \small Number of data        & \small 70,000 & \small 98,528   & \small 67,557 & \small 20,000 & \small 11,908 & \small 1,513 & \small 41,127  \\
     \small Feature dimension     & \small 784 & \small 100   & \small 126 & \small 16  & \small 2,052 & \small 2,048 & \small 2,048  \\
    \bottomrule
  \end{tabular}
  \end{adjustbox}
  \vspace{-2mm}
\end{table*}

\subsection{Synthetic Dataset}

The aim of this section is to provide an illustrative example of why CPT achieves better performance when compared to other decision tree based algorithms. To that end, we consider a dataset of 2,000 points, as shown in Figure \ref{fig:syn}. The data samples are independent draws of the uniform distribution on a two-dimensional space as $[-1, 1]\times [-1, 1]$, with the data points labeled as red if they lay between two concentric circles, and blue otherwise.

We compare CPT with LCN \citep{lee2019locally}, which is a  state-of-the-art oblique tree method, and CART \citep{breiman1984classification}, which is a canonical axis-aligned tree algorithm. Figure \ref{fig:syn} from left to right shows the decision boundaries of CART, LCN, and CPT, respectively. CART and LCN are trained for three different maximum depth parameters 2, 6, and 10; the corresponding plots are shown from the top to bottom. Due to the large number of nodes in LCN, we do not directly plot its decision boundaries. Instead, we illustrate the regions assigned to each of its leaf nodes with different shades of gray, and the darker the grey, the more red labels in that region. It is worth noting that both LCN and CART are restricted to partition the feature space into disjoint convex polytopes and assign each region to a leaf node. However, CPT does not have such a limitation, and each region can be the result of applying any set of logical operations on a set of convex polytopes. Figure~\ref{fig:syn} clearly shows that both LCN and CART need a large depth to successfully classify the data, while CPT only needs two splits. To be more specific, the AUC results for
CART are $[0.5, 0.708, 0.897]$ with $[3, 7, 11]$ leaves and the AUC results  for LCN are $[0.681, 0.773, 0.959]$ with $[3, 35, 479]$ leaves.

By contrast, CPT achieves the AUC of $0.962$ (the highest score) with just 3 leaf nodes at maximum depth 2. 

Figure~\ref{fig-c} shows that CPT uses a heptagon (7-sided 2D convex polytope) for each split. However, this number was not fixed at training time. We only limit the maximum number of polytope sides to $K=50$, the gamma process truncation level at each node. The model owes this adaptive shrinkage to the gamma process prior, which improves the simplicity and interpretability of~the~model.

\subsection{Classification and Regression}

We evaluate the performance of CPT for regression, binary classification, and multi-class classification tasks. For regression and binary classification, we conduct experiments on datasets from MoleculeNet~\citep{wu2018moleculenet}. We follow the literature to construct features \citep{wu2018moleculenet, lee2019locally} and use the same training, validation, and testing split as~\citet{lee2019locally}. For multi-class classification, 
we perform experiments on four benchmark datasets from LibSVM~\citep{chang2001libsvm}, including MNIST, Connect4, SensIT, and Letter. We employ the provided training, validation, and testing sets when available; otherwise, we create them under the criterion specified in previous works \citep{norouzi2015efficient, hehn2019end}. The datasets statistics are summarized in Table~\ref{tab:statistics}.

\begin{table*}[th]
  \caption{The performance of decision tree algorithms on regression (RMSE, lower is better) and classification (ACC or AUC, higher is better). The superscripts \tstar, \tdag, and \tddag  is used to to show the source of the quoted results as \citep{zharmagambetov2019experimental}, \citep{lee2019locally}, and \citep{hehn2019end}}\label{tab:results_all}\vspace{-2mm}
\begin{adjustbox}{width=1\columnwidth,center}
\centering
 \begin{tabular}{lccccccccccccc}
    \toprule
    \small Dataset     & \small {CART} & \small {HHCART} & \small {GUIDE} & \small {CO2 } & \small { FTEM} & \small {TAO }& \small {LCN } & \small {CPT ($K$=1)} & \small {CPT}\\
    \midrule
    \midrule
    \small MNIST\scriptsize(ACC)   & \small 88.05{\scriptsize{$\pm$0.02}}\tstar  & \small  90.1{\scriptsize{$\pm$1.2}}  & \small   78.52{\scriptsize{$\pm$0.20}}\tstar  & \small 90{\scriptsize{$\pm$-}}\tstar  & \small 96.12{\scriptsize{$\pm$-}}\tddag  & \small 94.74{\scriptsize{$\pm$0.11}}\tstar  & \small 93.81{\scriptsize{$\pm$0.32}}  & \small 95.74{\scriptsize{$\pm$0.10}} & \small \textbf{97.01{\scriptsize{$\pm$0.20}}}\\
    \midrule
    \small SensIT \scriptsize(ACC)       & \small 81.71{\scriptsize{$\pm$0.01}}\tstar  & \small  -  & \small  79.25{\scriptsize{$\pm$0.33}}\tstar   & \small 82{\scriptsize{$\pm$-}}\tstar   & \small 81.61{\scriptsize{$\pm$-}}\tddag   & \small 85.12{\scriptsize{$\pm$0.20}}\tstar   & \small 84.06{\scriptsize{$\pm$0.32}}   & \small 84.87{\scriptsize{$\pm$0.27}}  & \small \textbf{85.77{\scriptsize{$\pm$0.55}} }\\
     \midrule
    \small Connect4 \scriptsize(ACC)   & \small 78.29{\scriptsize{$\pm$0.21}}\tstar & \small  -  & \small  72.01{\scriptsize{$\pm$0.36}}\tstar  & \small 78{\scriptsize{$\pm$-}}\tstar  & \small 80.51{\scriptsize{$\pm$-}}\tddag  & \small \textbf{81.09{\scriptsize{$\pm$0.39}}}\tstar  & \small 80.79{\scriptsize{$\pm$0.12}}  & \small 78.89{\scriptsize{$\pm$0.23}} & \small \textbf{81.16{\scriptsize{$\pm$0.51}}}\\

     \midrule
    \small Letter \scriptsize(ACC)        & \small 86.07{\scriptsize{$\pm$0.14}}\tstar & \small  83.1{\scriptsize{$\pm$0.3}}  & \small  82.65{\scriptsize{$\pm$0.9}}\tstar  & \small 87{\scriptsize{$\pm$-}}\tstar  & \small 86.31{\scriptsize{$\pm$0.21}}  & \small 89.15{\scriptsize{$\pm$0.88}}\tstar  & \small 89.64{\scriptsize{$\pm$0.75}}  & \small 84.13{\scriptsize{$\pm$0.32}}  & \small \textbf{90.51{\scriptsize{$\pm$0.81}}}\\

    \midrule
    \small PDBbind \scriptsize(RMSE)       & \small 1.573{\scriptsize{$\pm$0.00}}\tdag & \small  1.530{\scriptsize{$\pm$0.00}}\tdag & \small  -  & \small -  & \scriptsize Not Applicable  & \small -  & \small 1.508{\scriptsize{$\pm$0.017}}\tdag  & \small 1.453{\scriptsize{$\pm$-0.006}} & \small \textbf{1.413{\scriptsize{$\pm$0.004}}}\\

    \midrule
     \small Bace \scriptsize(AUC)       & \small 65.2{\scriptsize{$\pm$2.4}}\tdag & \small  54.5{\scriptsize{$\pm$1.6}}\tdag  & \small  -  & \small -  & \small 81.03{\scriptsize{$\pm$1.5}}  & \small 73.4{\scriptsize{$\pm$0.0}}\tdag  & \small \textbf{83.9{\scriptsize{$\pm$1.3}}}\tdag  & \small 82.06{\scriptsize{$\pm$2.4}} & \small \textbf{84.7{\scriptsize{$\pm$1.6}}}\\

    \midrule
     \small HIV \scriptsize(AUC)   & \small 54.4{\scriptsize{$\pm$0.9}}\tdag & \small  63.6{\scriptsize{$\pm$0.0}}\tdag  & \small  -  & \small -  & \small 71.09{\scriptsize{$\pm$1.0}}  & \small 62.7{\scriptsize{$\pm$0.0}}\tdag  & \small \textbf{72.8{\scriptsize{$\pm$1.3}}}\tdag  & \small 71.12{\scriptsize{$\pm$2.3}} & \small \textbf{73.1{\scriptsize{$\pm$1.4}}}\\

    \bottomrule
  \end{tabular}
  \end{adjustbox}
\end{table*}

\begin{table*}[th!]
  \caption{Number of leaves and (the depth of trees)}\vspace{-2mm}
  \label{tab:n_leafs}
\begin{adjustbox}{width=1\columnwidth,center}
  \begin{tabular}{lccccccccccccc}
    \toprule
    \small Dataset     & \small {CART} & \small {HHCART} & \small {GUIDE} & \small {CO2 } & \small { FTEM} & \small {TAO }& \small {LCN } & \small {CPT(K=1)} & \small {CPT}\\
    \midrule
    \midrule

     MNIST          & \small 805 (D19)\tstar & \small  -  & \small  \textbf{39} (D15)\tstar  & \small - (D14)\tstar  & \small 357 (D12)  & \small 178 (\textbf{D8})\tstar  & \small 65536 (D16)  & \small 501 (D11)  & \small 98 \textbf{(D8)}\\
     
    SensIT         & \small 152 (D12)\tstar & \small  -  & \small  \textbf{24} (D11)\tstar  & \small - (\textbf{D6})\tstar  & \small - (D7)\tddag   & \small 69 (D7)\tstar  & \small 256 (D8)  & \small 81 (D8) & \small 36 (\textbf{D6})\\
     
     Connect4          & \small 5744 (D33)\tstar & \small  -  & \small  27 (D18)\tstar  & \small - (D16)\tstar  & \small 257 (D10)  & \small 210 (D8)\tstar  & \small 65536 (D16)  & \small 342 (D10) & \small \textbf{4 (D2)}\\     

     Letter          & \small 1580 (D27)\tstar & \small  -  & \small  673 (D28)\tstar  & \small - (D12)\tstar & \small 543 (D16) & \small 1078 (\textbf{D11})\tstar  & \small 65536 (D16)  & \small 705 (D17)  & \small \textbf{463} (\textbf{D11})\\
     
     PDBbind          & \small - & \small  -  & \small  -  & \small -  &  \scriptsize Not Applicable  & \small -  & \small 2048 (D11)\tdag  & \small 16 (\textbf{D4}) & \small \textbf{15} (\textbf{D4})\\
     
     Bace          & \small - & \small  -  & \small  -  & \small -  & \small 23 (D8)  & \small -  & \small 4096 (D12)\tdag  & \small \textbf{19} (D9)  & \small 21 (\textbf{D7}) \\
     
     HIV          & \small - & \small  -  & \small  -  & \small -  & \small 21 (D7)  & \small -  & \small 128 (D7)\tdag  & \small 25 (D7)  & \small \textbf{24} (\textbf{D5})\\
     

    \bottomrule
  \end{tabular}
  \end{adjustbox}
  \vspace{-2mm}
\end{table*}

\subsubsection{Compared Baselines} We evaluate the performance of CPT against several state-of-the-art decision tree methods, including \textsc{FTEM} (\citet{hehn2019end}, ``End-to-end learning of decision trees and forests''), \textsc{Tao} (\citet{carreira2018alternating}, ``Oblique decision trees trained via alternating optimization''), and \textsc{LCN} (\citet{lee2019locally}, ``Oblique decision trees from derivatives of ReLU networks''). We also consider several additional baselines, including \textsc{Cart} \citep{cart84}, \textsc{Hhcart} \citep{wickramarachchi2016hhcart},  \textsc{GUIDE} \citep{loh2014fifty}, and \textsc{CO2} \citep{norouzi2015efficient}. 
Moreover, to empirically show the importance of flexible boundaries, we also added a baseline \textsc{CPT}, where $K=1$, with hyperplane~cuts.

Our algorithm is implemented in PyTorch and can be trained by gradient-based methods. We use Adam \citep{kingma2014adam} optimization for inferring the tree split parameters. A 10-fold cross-validation on the combined train and validation set is used to learn the hyperparameters, namely the maximum number of polytope sides, number of training epochs, learning rate, and batch-size. However, we decide the depth of the tree based on the performance of CPT on the validation set during training, which can be perceived as early stopping for trees. Following the literature,
we use the Area Under the receiver operating characteristic Curve (AUC) on the test set as the evaluation metric for binary classification, accuracy (ACC) for multi-class classification, and root-mean-squared error (RMSE) for regression. 

Finally, we report the average and standard error of each method's performance by repeating our experiments for 10 random seed initializations. More details about our implementation and the exact values of hyperparameters for each dataset are presented in the Appendix. Our code to reproduce the results is provided at \url{https://github.com/rezaarmand/Convex_Polytope_Trees}.

\subsubsection{Experimental Results} Tables \ref{tab:results_all} and \ref{tab:n_leafs} present the results for a variety of decision tree based algorithms. Some results are quoted from previous works \citep{lee2019locally, hehn2019end, zharmagambetov2019experimental}. 
The depth and leaf numbers are averaged over 10 repetitions of training, and then rounded. 
For some large datasets, namely MNIST, SensIT, Connect4, and Letter, we fix the depth parameter as opposed to adaptively tuning it based on the validation set on each run. 
The reason for some missing values in the Table 2 is some methods like TAO did not provide their code, so we could not provide their performance on datasets they had not experimented. Regarding the hyper-parameter tuning for CPT, the total number of different hyper-parameter tuning setups for each dataset was less than 25 (5*5) cases. For the baseline methods, we quote the best results tuned and reported by the authors (e.g., LCN, according to their paper, does at least (3*11) hyperparameter tuning). Also, for TAO, we report the best results by the authors.
For each dataset, the best result and  those with no statistically significant difference (by using two sample $t$-test and $p$-value of 0.05)~are~highlighted.

From the results, it is evident that the added flexibility in splitting rules combined with an efficient training algorithm allows CPT to outperform the baseline algorithms. Our method achieves the state-of-the-art performance, while, notably, using significantly shallower trees. For instance, CPT obtains the best performance in Connect4 with only depth 2 and 4 leaves, while other methods need a depth~of~at~least~8. 

It also improves the regression performance on the PDBbind dataset by a large margin. Although LCN achieves competitive results in terms of accuracy on some datasets, it needs to grow the tree's size exponentially, significantly sacrificing the model interpretability. That is mainly because LCN, in contrast to our model, always learns a complete tree and generally needs to have a considerably large depth to achieve competitive results. For instance, consider its enormous size when trained on MNIST and Connect4 in Table \ref{tab:n_leafs}.

\section{CONCLUSION}
We propose convex polytope trees (CPT) as a generalization to the class of oblique trees that improves their accuracy and shrinks their size, which consequently provides more robust predictions. CPT owes its performance to two main components: flexible decision boundaries and an efficient training algorithm. The proposed training algorithm well addresses the challenge to learn not only the parameters of the tree but also its structure. Moreover, we demonstrate the efficacy and efficiency of CPT on a variety of tasks and datasets. 
The empirical successes of CPT show promise for further research on other interpretable generalizations of decision boundaries. This can lead to a significant performance gain for the family of decision tree models. Another promising direction for future work is investigating the combination of CPT with various ensemble methods, such as boosting.

%
%
%
%

%
%

%
%

%
%
%
\bibliographystyle{abbrvnat}
\bibliography{tree_ref.bib}

\appendix
\onecolumn
\singlespacing
\section{PROOFS}
In this section, we show why the splitting function at each internal node results in a convex set confined within a convex polytope. We start by proving a lemma which is needed to prove the main theorem.

\begin{lemma}\label{lemma}
For any $\{r_i,\betav_i\}_{i=1}^K$, such that $r_i \in \mathbb{R}_+ $ and $\beta_i \in \mathbb{R}^d $, the function:
\beq
g(\xv):= \sum_{i=1}^{K}r_i \ln{(1+e^{\betav_i'\xv})}
\eeq
is convex over its domain $\mathbb{R}^d$.
\end{lemma}
\begin{proof}
Since the sum of convex functions is also convex, it suffice to show each term of $g$ is a convex function. We demonstrate this by using the following theorem: ``A function is convex iff its second derivative is a positive semi-definite matrix over the domain.'' One can omit $r_i$'s in the following calculations because a positive scalar does not change the convexity.

The fist derivative of each term is:
\beq
\frac{\partial \ln{(1+e^{\betav_i'\xv})} }{\partial \xv} = \frac{e^{\betav_i'\xv}}{e^{\betav_i'\xv} + 1}\betav_i'
\eeq
and by taking the derivative of the above vector, we will have:
\beq
\frac{\partial^2 \ln{(1+e^{\betav_i'\xv})} }{\partial \xv^2} = \frac{e^{\betav_i'\xv}}{(e^{\betav_i'\xv} + 1)^2}\betav_i\betav_i'
\eeq

where $\betav_i\betav_i'$ is a matrix in $\mathbb{R}^d \times \mathbb{R}^d$. Since the scalar $ \frac{e^{\betav_i'\xv}}{(e^{\betav_i'\xv} + 1)^2}$ is positive for any $\xv$, we just need to show the matrix $\betav_i\betav_i'$ is positive semi-definite. To that end, we prove for any $\vv \in \mathbb{R}^d$:
$$\vv' (\betav_i\betav_i')\vv\geq 0.$$
And, that can be shown by:
$$\vv' (\betav_i\betav_i')\vv = (\vv' \betav_i)\cdot(\betav_i'\vv) = (\betav_i'\vv)' \cdot (\betav_i'\vv)= \norm{\betav_i'\vv}^2\geq 0. $$

Therefore the proof of the lemma is complete.

\end{proof}

\begin{theorem}
For any $\{r_i,\betav_i\}_{i=1}^K$, such that $r_i \in \mathbb{R}_+ $ and $\beta_i \in \mathbb{R}^d $, let: 
$$
A_{\textit{left}}=\{ \xv \given \xv \in \mathbb{R}^d,  ~~ \textstyle f(\xv)\leq q_{\textit{thr}} \} ,
$$
where:
\beq
\textstyle f(\xv) = 1-e^{-\sum_{i=1}^{K}r_i \ln{(1+e^{\betav_i'\xv})}},
\eeq
then $A_{\textit{left}}$ is a convex set, confined with a convex polytope.
\end{theorem}
\begin{proof}

We start by showing $A_{\textit{left}}$ is a convex set. Note that, due to the duality
\beq \label{eq:dual}
\xv \in A_{\textit{left}} \iff f(\xv) \leq q_{\textit{thr}}
\eeq
By the definition of a convex set, we just need to prove the following:
\beq \label{eq:state}
\forall t \in [0, 1], \forall \xv_1, \xv_2 \in \mathbb{R}^d  ~~~\text{if}~~~ f(\xv_1), f(\xv_2) \leq q_{\textit{thr}} \implies  f(t \xv_1 + (1-t) \xv_2) \leq q_{\textit{thr}}. 
\eeq

Let $g(.)$ and $q_{\textit{thr}}^*$ be:
\beq
g(\xv) := -\ln{(1-f(\xv))} = \sum_{i=1}^{K}r_i \ln{(1+e^{\betav_i'\xv})}, \quad q_{\textit{thr}}^*:= -\ln{(1-q_{\textit{thr}})}.
\eeq

Since $-\ln{(1-a)}$ is monticaly increasing with respect to $a$, replacing $f$ by $g(.)$ and $q_{\textit{thr}}$ by $q_{\textit{thr}}^*$ in (\ref{eq:state}), results in a mathematically equivalent expression. Now, we can prove the new statement using Jensen's inequality. To be more specific, based on Lemma \ref{lemma} ($g$ is convex) and Jensen's inequality, we have:
\beq 
\forall t \in [0, 1], \forall \xv_1, \xv_2 \in \mathbb{R}^d  ~~~~~~ g(t \xv_1 + (1-t) \xv_2) \leq t g(\xv_1)+ (1-t) g(\xv_2) . 
\eeq
So if $g(\xv_1), g(\xv_2) \leq q_{\textit{thr}}^*$:
$$g(t \xv_1 + (1-t) \xv_2) \leq t g(\xv_1)+ (1-t) g(\xv_2) \leq t q_{\textit{thr}}^*+ (1-t) q_{\textit{thr}}^* = q_{\textit{thr}}^*  $$
proving $A_{\textit{left}}$ is convex.


We are just remained with showing $A_{\textit{left}}$ is confined within a convex polytope. This can be shown by:
\begin{align*}
    f(\xv) \leq q_{\textit{thr}} \iff g(\xv) \leq q_{\textit{thr}}^* & \implies \forall i \in [1:K], ~~~ r_i \ln{(1+e^{\betav_i'\xv})} \leq {q_{\textit{thr}}^*} \\
    & \iff \forall i \in [1:K], ~~~ \betav_i'\xv \leq \ln{(e^{\frac{q_{\textit{thr}}^*}{r_i}} -1)}  
\end{align*}
which completes the proof.
\end{proof}

\section{ADDITIONAL DETAILS ON 
EXPERIMENTAL SETTINGS}

As mention in the paper, we train CPT in a probabilistic manner and switch to a deterministic tree at test time. To make the transition smoother, we conduct annealing during training. To be more specific, we transform the probability function $f(\xv)$ at each node to $f_{\lambda_t}(\xv)$, where:

\beq
f(\xv) = 1-e^{-\sum_{i=1}^{K}r_i \ln{(1+e^{\betav_i'\xv})}} \quad \text{and} \quad~~~~ f_{\lambda_t}(\xv) :=\frac{1}{1+(\frac{1-f(\xv)}{1- p_0})^{\lambda_t} }
\eeq

Larger $\lambda_t$ results in a sharper change of probability from 0 to 1, and $p_0$ controls where that change happens. During training, we gradually increase $\lambda_t$ to make the gap between probabilistic and deterministic tree progressively smaller. We also learn $p_0$ like other parameters of the model using SGD. Notice, the change of $f$ to $f_{\lambda_t}$ keeps the mathematical and geometrical interpretation of CPT intact. That is because any thresholding of $f_{\lambda_t}$ has an equivalent counterpart for $f$ since $f$ and $f_{\lambda_t}$ are strictly monotonic function of each other.

\section{ALGORITHM}

 \begin{algorithm}[H]
\caption{Stochastic gradient descent training of the tree splitting parameters for classification task.\newline\textbf{Input:} Data $\{(\xv^{(n)}, y^{(n)})\}_{n=1:N}$, initial tree $\mathcal{T}^{(0)}$ from GreedyTopology-Leaner algorithm, maximum number of polytope sides $K$, hyper-parameters of the gamma process prior $a_0, b_0, c_0, \gamma_0$}

\begin{algorithmic}
\label{alg:SGD}
\FOR{number of training iterations}
    \STATE{$\bullet$ Sample a batch of $m$ data samples $\{ (\xv_1, y_1), \dots, (\xv_m, y_m) \}$ }
    \STATE{$\bullet$ Send the data samples to the current probabilistic tree to get probability of each samples being in any leaf $P_{\thetav^{(t)}}(\ell_n = \ell \given \xv_n)$ for all  $n \in [1:m], \ell \in  S_{\textit{leaf}}$ using Eqn. (6).}
    \STATE{$\bullet$ For each leaf, calculate the vector $\hat{\piv}^{\ell}= [\hat{\pi}^{\ell}_{1},\dots, \hat{\pi}^{\ell}_{C}]$ which represent for the current batch what is the proportion of each data label in the leaf $\ell$. $$
\hat{\pi}^{\ell}_{j} = \frac{\sum_n \mathbf{1}_{[y_n=j]} P_{\thetav^{(t)}}(\ell \given \xv_n)}{\sum_n P_{\thetav^{(t)}}(\ell \given \xv_n)} $$ }

    \STATE{$\bullet$ Estimate the entropy of the data labels $\Ymat$, conditioned on the leaf id $\Lmat$ :$$
    \mathcal{\hat{H}}(\Ymat \given \Lmat) := -\sum\nolimits_{ \ell \in  S_{\textit{leaf}}}  \left[\sum\nolimits_{j=1}^C \hat{\pi}^{\ell}_{j} \log (\hat{\pi}^{\ell}_{j})\right] \times {\sum_{n=1}^{m} P_{\thetav^{(t)}}(\ell \given \xv_n)}
    $$
    }
    \STATE{$\bullet$ Calculate the regularization term related to the gamma process prior:

    \begin{dmath*}
    \mathcal{L}_{\textit{reg}}=\textstyle \sum_{\nu\in S_{\textit{branch}}} \left[ \sum_{k=1}^K\left(-(\frac{\gamma_0}{K}-1)\ln r_k^{(\nu)} + c_0  r_k^{(\nu)}\right)+\textstyle(a_0+1/2)\sum_{j=0}^d\sum_{k=0}^K [ \ln(1+\beta^{(\nu) 2}_{jk}/(2b_{0}) ) ]\right]
    \end{dmath*}

}

    \STATE{$\bullet$ Update the tree parameters by descending their stochastic gradient:
        \[
            \nabla_{\thetav^{(t)}} \left( \mathcal{\hat{H}}(\Ymat \given \Lmat) + \mathcal{L}_{\textit{reg}} \right)
        \]}
  \ENDFOR
  \STATE{$\bullet$ Send all the data samples to the to the tree, and for each leaf set its parameter as $\hat{\piv}^{\ell}$ }
  \STATE{$\bullet$ Use the threshold of 0.5 for the probabilistic decision function at each branch node to achieve deterministic~tree}
  
\end{algorithmic}
\end{algorithm}

\end{document}